%% file: mixingbandit-ARXIV.tex
\documentclass[10pt,a4paper]{article}
\usepackage{amsmath,amsthm,amsfonts,dsfont}
\usepackage{graphicx} 
\usepackage{vmargin}
\usepackage{amssymb}
\usepackage{xcolor}
\usepackage{xspace}
\usepackage{algorithm,verbatim}
\usepackage{algpseudocode}
\usepackage{verbatim}
\usepackage{enumitem}
\usepackage{authblk}

\usepackage{url}
\usepackage{bbm}

\usepackage[numbers]{natbib}

\author{
}

\theoremstyle{plain}
\newtheorem{theorem}{Theorem}
\newtheorem{proposition}{Proposition}

\theoremstyle{definition}
\newtheorem{definition}{Definition}

\newcommand{\expectation}{\mathbb{E}}

\newcommand{\proba}{\mathbb{P}}

\newcommand{\regret}{{\cal R}}

\newcommand{\indicator}[1]{\mathbbm{1}_{\left[#1\right]}}
\newcommand{\bfZ}{{\bf Z}}
\newcommand{\realset}{\mathbb{R}}
\newcommand{\Zspace}{\mathbb{Z}}
\newcommand{\Nspace}{\mathbb{N}}

\DeclareMathOperator*{\argmax}{argmax}
\newcommand{\Xseq}{\underline{X}}
\newcommand{\xseq}{\underline{x}}
\newcommand{\liva}[1]{\textcolor{blue}{#1}}
\newcommand{\mand}{\;\wedge\;}
\newcommand{\mor}{\;\vee\;}
\newcommand{\event}{{\cal E}}

\usepackage{float}

\newcommand{\bpf} {\noindent{\sc Proof} : }
\newcommand{\epf} {\hfill$\square$\vspace{.5cm}}
\newcommand{\E} {\mathbb{E}}

\renewcommand{\P} {\mathbb{P}}

\newcommand{\N}{\mathbb{N}}
\newcommand{\R}{\mathbb{R}}

\newcommand{\m}{\mathcal{M}}
\newcommand{\bfbeta}{\boldsymbol{\beta}}

\newcommand{\Uspace}{\mathcal{U}}

\usepackage{exscale}
 

\makeatletter
\newenvironment{equationsize*}[1]{%
  \skip@=\baselineskip 
  #1%
  \baselineskip=\skip@ 
  \equation
}{\nonumber\endequation \ignorespacesafterend} 
\makeatother

 

\makeatletter
\newenvironment{alignsize*}[1]{%
  \skip@=\baselineskip 
  #1%
  \baselineskip=\skip@ 
  \start@align\@ne\st@rredtrue\m@ne
}{\endalign\ignorespacesafterend} 
\makeatother

\newcommand{\captionfonts}{\small}

\makeatletter  
\long\def\@makecaption#1#2{%
  \vskip\abovecaptionskip
  \sbox\@tempboxa{{\captionfonts #1: #2}}%
  \ifdim \wd\@tempboxa >\hsize
    {\captionfonts #1: #2\par}
  \else
    \hbox to\hsize{\hfil\box\@tempboxa\hfil}%
  \fi
  \vskip\belowcaptionskip}
\makeatother   
 
\title{
Stationary Mixing Bandits
}

\author[1]{Julien Audiffren}
\author[2]{Liva Ralaivola}
\affil[1]{CMLA, ENS Cachan,Cachan,France}
\affil[2]{QARMA, Aix-Marseille Université, CNRS, LIF, Marseille, France}
\date{}

\begin{document}
\maketitle

\begin{abstract}
We study the bandit problem where arms are associated with stationary $\varphi$-mixing processes and where rewards are therefore dependent: the question that arises from this setting is that of recovering some independence by ignoring the value of some rewards. As we shall see, the bandit problem we tackle  requires us to address the exploration/exploitation/independence trade-off. To do so, we provide a UCB strategy together with a general regret analysis for the case where the size of the independence blocks (the ignored rewards) is fixed and we go a step beyond by providing an algorithm that is able to compute the size of the independence blocks from the data. Finally, we give an analysis of our bandit problem in the restless case, i.e., in the situation where the time counters for all mixing processes simultaneously evolve.  


\end{abstract}

\input{introduction}

\input{background}
\input{easybandits}

\input{vectorbandit}
\input{restless}
\input{conclusion}

\newpage
\bibliographystyle{unsrt}
\bibliography{mixingbandits}

\appendix
\newpage
\input{appendix}

\end{document}

%% file: introduction.tex

\section{Introduction}
\label{sec:introduction}

\paragraph{Bandit with mixing arms.} The bandit problem consists in an agent who has to choose at each step between $K$ arms. A stochastic process is associated to each arm, and pulling an arm  produces a reward which is the realization of the corresponding stochastic process. The objective of the agent is to maximize its long term reward. 
It is classically assumed that the stochastic process associated to each arm is a sequence of independently and identically distributed (i.i.d) random variables (see, e.g.~\cite{LaiR85AM}). In that case, the challenge the agent has to face is the well-known exploration/exploitation problem: she has to simultaneously make sure that she collects information from all arms to try to identify the most rewarding ones ---this is \textit{exploration}--- and to maximize the rewards along the sequence of pulls she performs ---this is \textit{exploitation}.
Many algorithms have been proposed to solve this trade-off between exploration and exploitation~\cite{AuerCBF02MLJ,AuerO10MH,bubeckCB12Regret,LaiR85AM}. We propose to go a step further than the i.i.d setting and to work in the situation where the process associated with each arm is a stationary $\varphi$-mixing process and the rewards are thus   dependent from one another, but with a strength of dependence that weakens over time. From an application point of view, this is a reasonable dependence structure: if a user clicks on some ad (a typical use of bandit algorithms) at some point in time, it is very unlikely that she will click again on this ad in the near future.
As it shall appear in the sequel, working with such dependent observations poses the question of how informative are some of the rewards with respect to the value of an arm since, because of the dependencies and the high correlation between close-by (in time) rewards, they might not reflect the true 'value' of the arms. However, as the dependencies weaken over time, some kind of independence might be recovered if rewards are ignored. This actually requires to deal with a new trade-off exploration/exploitation/independence
that need be precisely handled.

\paragraph{Rested and Restless case.} A closely related setup that addresses the bandit problem with dependent rewards is when they are distributed according to Markov processes, such as Markov chains and Markov decision process (MDP) \cite{Ortner:14maba,TekinL12}, where the dependences between rewards are of bounded range, which is what distinguishes those works with ours.
Contributions in this area study two settings, that we will analyze as well: the \textit{rested} case, where the process attached to an arm evolves only when the arm is pulled, and the {\em restless} case, where all processes simultaneously evolve at each time step. 





\paragraph{Contributions and structure of the paper.}
We define the notion of a $\varphi$-mixing bandit and its regret (Section~\ref{sec:background}), we provide a general analysis and an algorithm to solve the rested case where the size of independence blocks is fixed (Section~\ref{sec:easybandits}). We provide another approach where these sizes are computed from the data by introducing another algorithm (Section~\ref{sec:vectorbandit}). Finally, in Section~\ref{sec:restless}, we provide an algorithm and a regret analysis to deal with the restless case.

%% file: background.tex

\section{Overview of the problem}
\label{sec:background}

Let $(\omega, \mathcal{F}, \P)$ be a probability space.
We recall the definitions of stationary and of $\varphi$-mixing processes: 
\begin{definition}[Stationarity]
A sequence of random variables $\underline{X}=\{X_t\}_{t=-\infty}^{+\infty}$ is {\em stationary}
if, for any $t$ and nonnegative integer $m$ and $s$, the random subsequences $(X_t,\ldots,X_{t+m})$ and $(X_{t+s},\ldots,X_{t+m+s})$
are identically distributed.
\end{definition}

\begin{definition}[$\varphi$-mixing process] Let $\underline{X}=\{X_t\}_{t=-\infty}^{+\infty}$ be a stationary
sequence of random variables. For any $i,j\in\mathbb{Z}\cup\{-\infty,+\infty\}$, let $\sigma_i^j$ denote
the $\sigma$-algebra generated by the random variables $X_t$, $i\leq t\leq j$. Then, for any positive
integer $n$, the $\varphi$-mixing coefficient $\varphi(n)$ of the stochastic process $\underline{X}$ is defined as
\begin{equation}
\label{eq:phicoefficient}
\varphi(n)=\sup_{t,A\in\sigma_{t+k}^{+\infty},B\in\sigma_{-\infty}^t}\left|\proba\left[A|B\right]-\proba\left[A\right]\right|.
\end{equation} 
$\underline{X}$ is said to be $\varphi$-mixing if $\varphi(n)\rightarrow 0$ as $n\rightarrow \infty$.
\end{definition}

We are interested in the problem of sampling from a $K$-armed $\varphi$-mixing bandit.
In our setting, pulling arm $k$ at time $t$ provides the agent with a realization of
the random variable $X^k_{\tau_k(t)}$, where $\tau_k(t)=t$ in the restless case and $\tau_k(t)$ is the number of times arm $k$ was pulled in the rested case, and where the family $(X_{t}^k)_{t\geq 1}$ satisfies the following hypotheses : 
\begin{enumerate}
\item $\forall k,\; (X_{t}^k)_{t\in\Zspace}$ is stationary;
\item the sequences $(X_{t}^k)_{t\in\Zspace}$ are
$\varphi$-mixing;
\item each $X^k_1$ takes values in a discrete finite set.
\end{enumerate}

This setting assumes the possibility of long-term
dependencies between the rewards output by the arms. It is important to note that by definition of the $\varphi$-mixing processes, the amount of dependence decreases with time. Hence, as evoked earlier, in order to choose which arm to pull, the agent is forced to
address the exploration/exploitation/independence trade-off 
where {\em independence} may be partially recovered by ignoring some rewards so as to make computations on data that are distant in time, i.e. data that are not too correlated (thanks to the mixing property).



It is critical to note that unlike in the i.i.d. framework, Hoeffding inequality cannot be applied in this setting, thus the widely used upper confidence bound (UCB) algorithms cannot be used here. In the case of stationary $\varphi$-mixing distributions, we have the following theorem from \cite{kontorovich08concentration}. 
\begin{theorem}[\cite{kontorovich08concentration,MohriR10jmlr}]
\label{th:kontorovitch}
Let $\psi:\mathcal{U}^m\rightarrow\realset$ be a function defined over a countable space $\mathcal{U}$, and $\underline{X}$ be a stationary $\varphi$ mixing process. 
If $\psi$ is $l$-Lipschitz with respect to the Hamming metric for some $l>0$, then the following holds
for all $t>0$:
\begin{equation}
\label{eq:kontorovitch}
\proba_{\underline{X}}\left[\left|\psi(\underline{X})-\expectation{\psi(\underline{X})}\right|>t\right]\leq 2\exp\left[-\frac{t^2}{2ml^2\|\Lambda_m\|_\infty^2}\right],\end{equation}
where $\|\Lambda_m\|_\infty\leq 1+2\sum_{k=1}^m\varphi(k)$.
\end{theorem}

In the following, we consider a more general framework than the one usually encountered in the bandit literature. Instead of looking at cumulative gains, we look at rewards computed according to Lipschitz functions meeting the requirements of the concentration inequality stated in Theorem~\ref{th:kontorovitch}.

More precisely, we suppose that we have $K$ known families $(\psi^k_{t})_{t\geq 1}$ $k=1,\ldots ,K$ of functions such that for every non-negative integer $m$:
\begin{enumerate}
\item $\psi^k_{m}:\Uspace^m\rightarrow\realset$ accounts for the reward associated with $m$ consecutive outcomes of arm $k$;
\item $\psi^k_m$ is $1$-Lipchitz with respect to the Hamming metric. 
\end{enumerate}

In the following sections, we use $m$ to identify the reward functions $(\psi_m^k)_{1\leq k\leq K}$ we want to optimize, $b$ to refer to the size of the independence blocks and we consider blocs of $s$ trials. We study three different scenarios. In the next section, we will present a general algorithm and regret analysis in the rested case where $m$ is fixed and  $s\doteq m+b$. In Section \ref{sec:vectorbandit}, with an additional hypothesis on the $\varphi$-mixing processes, we will take another approach to the rested case when $m+b \mid s$ (i.e. $s$ is a multiple of $m+b$)  by including the notion of dependency into the regret. Finally, in Section~\ref{sec:restless}, we  present a general algorithm and regret analysis for the restless case.


%% file: easybandits.tex

\section{Mixing Bandits in the Rested Case}
\label{sec:easybandits}
\paragraph{Regret analysis.}
Here, we are going to analyze the situation where 
$m$ and $b$ are
 {\em fixed}.
Our goal is to show that a simple algorithm derived from UCB that works by making blocks  $s\doteq m+b$ consecutive trials on each arm has low regret, with a notion of regret that we define in the sequel.
The running time index is therefore of the form of $st$, with $t$ the number of times arm selection has been performed,  and the sequence accessed to are such as $(X_{st}^k,\ldots,X_{(s+1)t-1}^k).$
The reward that is accessed to at time $t$ (with a slight abuse of notation that makes us use $t$ as the time index) when pulling arm $k$ does not make use of the full information provided by this sequence but instead is  $\psi_{m}^k(\underline{X}_{t,s,b}^k)$ where
\begin{equation}
\label{eq:Xseq}
\Xseq_{t,s,b}^k\doteq (X_{st}^k,\ldots,X_{(s+1)t-b-1}^k);
\end{equation}
this means that only the first $m$ points from $(X_{st},\ldots,X_{(s+1)t-1})$ are taken advantage of. 

Given $\tau$, the total number of trials of $s$-blocks, the regret that we are going to work with is
\begin{equation}
\regret\doteq \tau\mu_{\psi,m}^*-\sum_{k=1}^K\expectation\tau_k(\tau)\mu_{\psi,m}^k
\end{equation}
where: 
$\mu_{\psi,m}^*\doteq\max_{k=1,\ldots,K}\mu_{\psi,m}^k$,  $\quad \mu_{\psi,m}^k\doteq\expectation_{X_1^k,\ldots,X_m^k}\psi^k_m(X_1^k,\ldots,X_m^k)$
and $\tau_k(t)$ is the number of times arm $k$ has been chosen given that a total (i.e. over all arms) of $t$ pulls of $s$-blocks have been performed.

The arm selection procedure of the algorithm that we propose and dub Block-UCB, is depicted in Algorithm~\ref{alg:blockUCB}, where the function $\Lambda_k$ is defined as
\begin{equation}
\label{eq:lambda}
\Lambda_k(t)\doteq 1+2\sum_{r=1}^t\varphi^k(rb+(r-1)m).
\end{equation}
It is possible to show that Block-UCB has the following regret.
\begin{theorem}[Regret of Block-UCB]
\label{th:regret_blockUCB}
The regret of Block-UCB is bounded by
$$\sum_{k:\mu_{\psi,m}^k<\mu_{\psi,m}^*}\left(u_k\Delta_k+\frac{1}{\alpha-2}\right),$$
where the $u_k$'s are the solutions of the problems
\begin{equation}
\label{eq:uk}
u_k\Delta_k^2-8\alpha\Lambda_k^2(u_k)\log \tau=0,\;k=1,\ldots,K.
\end{equation}
\end{theorem}

The result of Theorem~\ref{th:regret_blockUCB} hinges on the derivation of a concentration inequality for each arm $k$ that relates the random variable
$
\frac{1}{\tau}\sum_{r=0}^{\tau-1}\psi_{m}^k\left(\underline{X}_{r,s,b}^k\right) 
$ to  $\mu_{\psi,m}^k.$ 
To establish this concentration inequality, we study the random variables $\Gamma_b^k(\underline{X}_{0,s,b}^k,\ldots,\underline{X}_{\tau-1,s,b}^k)$, defined for $b\geq 0$ as
\begin{equation}
\Gamma_b^k(\underline{X}_{0,s,b}^k,\ldots,\underline{X}_{\tau-1,s,b}^k)\doteq \frac{1}{\tau}\sum_{r=0}^{\tau-1}\psi_{m}^k(X_{rs}^k,\ldots,X_{(r+1)s-b-1}^k)-\mu_{\psi,m}^k.
\end{equation}	
The concentration inequality that we are going to use to prove our regret bound is the following:
\begin{theorem}
For all $\tau,k,b$, and assuming that $\psi_m^k$ takes value in $[0;1]$:
\begin{equation}
\label{eq:wbb_1}
\proba\left(\left|\Gamma_b^k(\underline{X}_{0,s,b}^k,\ldots,\underline{X}_{\tau-1,s,b}^k)\right|\geq \varepsilon\right)\leq \exp\left(-\frac{\tau\varepsilon^2}{2\Lambda_{k}^2(\tau)}\right),
\end{equation}
where $\Lambda_{k}$ is defined in~\eqref{eq:lambda}.
\end{theorem}
\begin{proof}
We make the proof for some arm $k$. We also assume that $s$ and $b$ are fixed and to lighten the notation, we drop the dependence on these variables when no confusion is possible: we use $\Xseq_{r}$ (resp. $\Gamma$) for $\Xseq_{r,s,b}^k$ (resp. $\Gamma_b$). The proof hinges on the fact that since $(X_t)_{t\geq 0}$ is a stationary mixing sequence with mixing coefficients $(\varphi(t))_{t\geq 1}$, $(\Xseq_{r,})_{r\geq 0}$ is a stationary mixing sequence with mixing coefficients $(\underline{\varphi}(q))_{q\geq 1}$ such that $\underline{\varphi}(q)\doteq\varphi(qb+(q-1)m)$ (see Proposition~\ref{prop:stationary_subsequence}, Appendix). To obtain the targeted result, we make use of the concentration inequality of Theorem~\ref{th:kontorovitch} with the function $\Gamma$ to get~\eqref{eq:wbb_1} and we exploit the observation that $\expectation\Gamma_b^k=0$, which results from the stationarity of $(X_t)_{t\geq 0}$.

Let $q$ be an integer in $\{0,\ldots,r-1\}$, $\xseq_0,\ldots,\xseq_{r-1}$ and $\xseq_q'$
blocks from $\Uspace^{s-b}$. Then:
\begin{align*}
\left|\Gamma(\xseq_0,\ldots,\xseq_q,\ldots,\xseq_{r-1})-\Gamma(\xseq_0,\ldots,\xseq_q',\ldots,\xseq_{r-1})\right| = \left|\frac{1}{\tau}(\psi_{s-b}(\xseq_q)-\psi_{s-b}(\xseq_q'))\right|\leq \frac{1}{\tau},
\end{align*}
which comes from the range of $\psi_{s-b}$ being $[0;1]$. $\Gamma$ is therefore $1/\tau$-Lipschitz with respect to the Hamming metric, which, combined with $(\Xseq_r)_{r\geq 0}$ being a $\varphi$-mixing sequence, gives~\eqref{eq:wbb_1}.
\end{proof}

\begin{algorithm}[t]
\caption{\label{alg:blockUCB}Main iteration of Block-UCB}
\begin{algorithmic}
\State Choose arm $I_t$ as 
$$I_t\in \argmax_k\frac{1}{t}\sum_{r=0}^{t-1}\psi_{m}^k\left(\underline{X}_{r,s,b}^k\right) +\Lambda_k(\tau_k(t-1))\sqrt{\frac{2\alpha\log t}{\tau_k(t-1)}}.$$
\end{algorithmic}
\end{algorithm}

The proof of the previous theorem uses the following more general result, that is of independent interest.
\begin{theorem}[General Regret]
\label{th:generic_regret}
Suppose that the arms we work with are such that
\begin{equation}
\label{eq:generic_concentration}
\forall k,\;\proba(\left|\hat{\mu}_{\tau}^k-\mu^k\right|\geq \varepsilon)\exp\left(-\theta_k(\tau)\gamma_k(\varepsilon)\right),
\end{equation}
where $\tau$ is the number of data the empirical mean $\hat{\mu}_{\tau}^k$ is computed on, and $\theta_k$ and $\varepsilon_k$ are increasing functions defined on $(0;+\infty]$.

Consider the regret defined by
$$R\doteq\tau\mu^*-\sum_{k=1}^K\expectation \tau_k(\tau)\mu_k$$
where $\tau_k(t)$ is the number of times a (suboptimal) arm $k$ has been chosen up to time $t$.

The $(\alpha,\theta,\gamma)$-UCB that chooses at iteration $t$ an arm $I_{t}$ according to 
\begin{equation*}
I_{t}\in\argmax\hat{\mu}_{\tau_k(t-1)}^k+\gamma^{-1}_k\left(\frac{\alpha}{\theta_k(\tau_k(t-1))}\log t\right)
\end{equation*}
has regret bounded by:
\begin{equation}
\label{eq:generic_regret}
\sum_{k:\mu_i<\mu^*}\left(\left\lceil\theta^{-1}_k\left(\frac{\alpha\log\tau}{\gamma_k(\Delta_k/2)}\right)\right\rceil\Delta_k+\frac{1}{\alpha-2}\right).
\end{equation}
\end{theorem}
\begin{proof}
Note that Theorem~\ref{th:regret_blockUCB} is a consequence of this theorem with $\theta_k(s)=s/\Lambda_k^2(s)$ and $\gamma_k(\varepsilon)=\varepsilon^2/2.$

The proof use the standard technique to prove the regret of UCB-like algorithms. Namely, at iteration $t$, if $I_{t}=i$ for $i$ not optimal, then one of the following events $\event_1^*(t), \event_2(i,t), \event_3(i,t)$ must occur
\begin{align}
\event_1^*(t)&\doteq\left\{ \hat{\mu}^*<\mu^*-\gamma_*^{-1}\left(\frac{\alpha}{\theta_*(\tau^*(t-1))}\log t\right)\right\},\label{eq:cond1}\\
\event_2(i,t)&\doteq\left\{\mu^i\leq\hat{\mu}^i-\gamma_i^{-1}\left(\frac{\alpha}{\theta_i(\tau_i(t-1))}\log t\right)\right\},\label{eq:cond2}\\
\event_3(i,t)&\doteq\left\{\tau_i(t-1)\leq\theta_i^{-1}\left(\frac{\alpha\log\tau}{\gamma_i(\Delta_i/2)}\right)\right\}.\label{eq:cond3}
\end{align}
Indeed, if none of the events occurs then (using $\Delta_i=\mu^*-\mu_i$)
\begin{equationsize*}{\small}\hat{\mu}^*+ \gamma_*^{-1}\left(\frac{\alpha}{\theta_*(\tau^*(t-1))}\log t\right) \stackrel{\text{\eqref{eq:cond1}}}{\geq}\mu_i+\Delta_i
\stackrel{\text{\eqref{eq:cond2}}}{>}\mu_i+2\gamma_i^{-1}\left(\frac{\alpha}{\theta_i(\tau_i(t-1))}\log \tau\right)
\stackrel{\text{\eqref{eq:cond3}}}{>} \hat{\mu}_i+\gamma^{-1}_i\left(\frac{\alpha}{\theta_i(\tau_i(t-1))}\log t\right)
\end{equationsize*}
where we have used that $t\mapsto\gamma_i^{-1}\left(a\log t\right)$ is an increasing function of $t$ on $[1;\infty)$ whenever $a>0$. This implies that $I_t\neq i$, which contradicts our working hypothesis.

If we let $u$ be defined as:
$$u\doteq\left\lceil\theta_i^{-1}\left(\frac{\alpha \log\tau}{\gamma_i(\Delta_i/2)}\right)\right\rceil,$$
then, for $i$ suboptimal, we have the following
\begin{align*}
\expectation\tau_i(\tau)&=\sum_{t=1}^{\tau}\expectation\indicator{I_t=i}\leq u + \sum_{t=u+1}^{\tau}\expectation\indicator{I_t=i\mand \neg\event_3(i,t)}\leq u + \sum_{t=u+1}^{\tau}\expectation\indicator{\event_1^*(t)\mor\event_2(i,t)}\\
&\leq u + \sum_{t=u+1}^{\tau}[\proba(\exists t:\event_1^*(t))+\proba(\exists t:\event_2(i,t))]
\end{align*}
Using the  union bound and Equation~\eqref{eq:generic_concentration}, both probabilities $\proba(\exists t:\event_1^*(t))$, $\proba(\exists t:\event_2(i,t))$ can be bounded from above by $1/t^\alpha$. Standard calculations allow us to get desired result~\eqref{eq:generic_regret}.
\end{proof}

\paragraph{Discussion.} Some observations must be made regarding the result of Theorem~\ref{th:regret_blockUCB}. First, as we used Theorem~\ref{th:generic_regret} to prove our regret bound, it is necessary for the result to hold for the functions $\theta_k\doteq\tau/\Lambda_k^2(\tau)$ to be increasing. In addition, the regret only makes sense if it is bounded, i.e. if the $u_k\Delta_k$ are bounded. Finally, if these conditions hold, it might be interesting to find, for some fixed horizon $\tau$, the value of $b$ that minimizes the regret. We now depicts common settings for the $\varphi^k$ that yield instructive results and that build upon the previous remarks. For the sake of conciseness, we will assume that all $\varphi^1=\ldots=\varphi^K$ and we use $\varphi$ to refer to the mixing coefficients.
\begin{description}
\item[Independent case.] if $\varphi=0$, i.e. we are in the independent case, and the $\theta_k$'s are naturally increasing. In addition, it is straightforward to observe that the best use of the data is achieved for $b=0$, i.e. each and every reward is used to estimate the quality of an arm.
\item[Case $\Lambda<+\infty$.] In that case, we are again back to a situation almost similar to the usual independent case. The $\theta_k$'s are increasing, the $u_k$ are well-defined and the regret as the usual $O(\sum_k\log\tau/\Delta_k)$ form.
\item[Algebraically mixing case.] Here, $\varphi(t)=\varphi_0t^{-p}$ for $p>1$, and a few calculations give
\begin{alignsize*}{\small}
\Lambda(\tau)&=1+2\varphi_0\sum_{r=1}^\tau\frac{1}{(rs-m)^p}\leq 1+2\varphi_0\left(1+\int_{1}^\tau\frac{1}{(rs-m)^p} dr\right)\\
&= 1+2\varphi_0 + \frac{2\varphi_0  }{s(p-1)} \left(\frac{1}{b^{p-1}}-\frac{1}{(\tau s-m)^{p-1}} \right)
\end{alignsize*}
Using this upper bound to find the $u_k$s as in~\eqref{eq:uk} and to solve for $b$ so that this bound is minimized provides a way to find a data-dependent $b$. Another (coarser) way to look at the algebraically mixing situation is not to optimize for $b$ and to consider that it is a particular case of the previous case, since $\sum_t \varphi(t) < \infty$, i.e. $\Lambda<+\infty$. This assumption is made in the rest of the paper. 
\end{description}

%% file: vectorbandit.tex
\section{$m+b \mid s$: Expressing the Independence Trade-Off in the Regret }
\label{sec:vectorbandit}
This section introduces another way of encoding the trade-off between exploration, exploitation and independence.
As before, we consider sequences of $s$ trials, but among the $s$ results, we seek to optimize the number $m$ of results we use to update the empirical value of the arms and the number $b$ of results we ignore in order to improve the independence between the considered realization of our random variables. 

In addition to the case $s=m+b$, we also consider the situation where $\beta(m+b)=s$ with $\beta \in \N$ and $\beta > 1$. 
The sequence of $s$ trials can then be interpreted as $\beta$ successive sequences of $m+b$ trials, and the value of this particular $(m,b)$ distribution is thus multiplied by $\beta$.

\subsection{Hypotheses and Regrets}
In this section we make the following additional assumption on the $\varphi$-mixing processes $(X^k_t)_{t\geq 0}$:
$$\forall 1\le k \le K,\quad \forall b \in \N ,\quad  M_k(b)\doteq 1+ \sum_{i\ge 1}\varphi^k(b(i+1))<+\infty.$$
Note that this is equivalent to the widely used assumption that the $\varphi^k(i)$ are summable over $i$ (see for instance the case of algebraically mixing sequence mentioned before). Also, note that $M_k(1)$ is an upper bound of $(\Lambda^k_m)_m$ which appears in Theorem~\ref{th:kontorovitch}, and  $M_k(\cdot)$ is a decreasing function such that $M_k(b) \ge 1.$

The setting is the following: at each step, the agent pulls an arm $s$ times, and has to choose how to split those $s$ elements between a meaningful part of $m$ elements, used to update empirical values of the arm, and the non-significant part of $b$ element, used to strengthen the independence between the variables. For each such combination $m+b \mid s$ and for each arm $k$, we define the value of the combination $(m,b,k)$ as:
\begin{equation}\label{eq:value}
\nu^k_{m,b} \doteq \frac{\beta_{m,b}}{M_k(b)} \mu^{k}_{m},
\end{equation}
where $\mu_{m}^k\doteq\expectation_{X_1^k,\ldots,X_m^k}\psi^k_m(X_1^k,\ldots,X_m^k)$ and $\beta_{m,b}=s /(m+b)$. This value explicitly shows the trade-off between independence (through $M_k(b)$) and exploitation (through the $\mu^k_{m,b}$). 
The value of an arm $k$ is then defined as the maximum value of the possible combination $(m,b,k)$, for $m+b|s$.
 
With this in mind, we define the  regret $\regret$ at time $t$ as :
\begin{equation}\label{eq:regretmulti}
 \regret\doteq \sum_{k=1}^K \E(T(k)) \left( \nu^{k^*}_{m^*,b^*} - \nu^{k}_{m_k^*,b_k^*}\right)
\end{equation}
where 
 $(k^*,m^*,b^*)= \arg\max_{(k,m,b)} \nu^{k}_{m_k,b_k},$  $(m_k^*,b^*_k)= \arg\max_{(m,b)} \nu^{k}_{m_k,b_k},$ and $T(k)$ denotes the number of times arm $k$ was pulled.
 
It is important to note that one of the main difference between \eqref{eq:regretmulti} and the classical formulation of regret from a multi arm bandit in the i.i.d. case is that in our setting, we are comparing the value of the best combination of the best arm with the value of the best combination of the pulled arm.


 \subsection{ Concentration inequality and algorithm}
  \begin{algorithm}[t]
\caption{\label{alg:blockUCBvar}Block-UCB with parameters $s$, $\alpha$ fixed}
\begin{algorithmic}
\State $t\leftarrow 0$, $\widehat{\psi}_{m,b,0}^k\leftarrow 0,\; k=1,\ldots,K, m+b \mid s$, $\tau_{k,0}\leftarrow 0,\; k=1,\ldots,K$
\For{$t=1,\ldots,\tau$}
\State Select arm $\widehat{k}$ with 
$\displaystyle \quad \widehat{k}\in\argmax_k  \max_{m+b \mid s }\frac{s}{m+b} \left( \widehat{\psi}_{m,b,t-1}^k+ \sqrt{\frac{2 \alpha (m+b) \log(t)}{s \tau_{k,t-1}}} \right)$
\State Update the block counts
$
\tau_{k,t}\leftarrow \tau_{k,t-1}+\delta_{k\widehat k},\; \forall k
$

\State Compute the values of the $\widehat{\psi}_{m,b,t}^k$, $\forall  m+b \mid s$
\begin{align*}
\widehat{\psi}_{m,b,t}^{\widehat{k}}&\leftarrow\frac{1}{M_k(b)\beta \tau_{\widehat{k},t}}\sum_{r=0}^{\beta(\tau_{\widehat{k},t}-1)}\psi_m^{\widehat{k}}\left(X_{ r(m+b)}^{\widehat{k}},\ldots,X_{(m+b) r+m-1}^{\widehat{k}}\right)\\
\widehat{\psi}_{m,t}^k&\leftarrow \widehat{\psi}_{m,t-1}^k,\; \text{for } k\neq \widehat{k}
\end{align*}
\EndFor
\end{algorithmic}
\end{algorithm}

We now introduce a concentration inequality tailored for the $\nu^k_{m,b}$ of \eqref{eq:value}.
\begin{proposition}
\label{pr:concentration}
Let $1 \le k \le K$, $1 \le m \le s$, $b=m-s$, $n \in \Nspace^*$, $\psi_m^k:\mathcal{U}^m\rightarrow\realset$ be a $1$-Lipschitz with respect to the Hamming metric function defined over a countable space $\mathcal{U}$. Suppose that the $\varphi^k(i)$ are summable over $i$, and let us define $M_k(b)= 1+ \sum_{i\in \Nspace^*} \varphi^k(i(b+1))$, and $\zeta_m^s: \Nspace \mapsto \Nspace$, $\zeta_m^s(t) =  t +b \lfloor (t-1)/m \rfloor $.
Then the following holds
for all $t>0$:
\begin{equation}
\label{eq:appliedconcentration}
\proba\left[ \frac{1}{M^k(b)}\left| \frac{1}{n}\sum_{i=0}^{n-1} \psi_m ^k\left(X^k_{\zeta_m^s(im+1)},\cdots,X^k_{\zeta_m^s((i+1)m)}\right) -\mu^k_m\right|>t\right]\leq 2\exp\left[-\frac{n t^2}{2}\right].
\end{equation}
\end{proposition}

\begin{proof}
This proposition naturally follows from the proof of Theorem \ref{th:kontorovitch} with the function $\phi_n^k= \frac{1}{n}\sum_{i=0}^{n-1} \psi_m ^k\left(X_{\zeta_m^s(im+1)},\cdots,X_{\zeta_m^s((i+1)m)}\right)$ which is $1/n$-Lipschitz with respect to the Hamming metric, and Proposition 5 ( see Appendix).
\end{proof}

Note that the upper bound on the probability that appears in \eqref{eq:appliedconcentration} is uniform over $(k,m,b)$. This is crucial to define our algorithm and to analyze its regret (more details in the next subsection).
We now introduce Algorithm~\ref{alg:blockUCBvar} that is designed for the particular setting of $\varphi$-mixing bandit problem just described. First, note that since the pair $(m,b)$ (with $m+b \mid s$) which gives the best result for each arm is unknown, the algorithm needs to compute an empirical estimator for each of these combinations for each arm. In other words, the algorithm needs to efficiently and simultaneously learn  both the best combination and the best arm.

\subsection{Regret Analysis}
In this subsection we provide an upper bound for the regret of Algorithm \ref{alg:blockUCBvar}.
\begin{proposition}\label{pr:regretms}
Let $\regret$ be the regret as defined in \eqref{eq:regretmulti} and let $\eta: \N \mapsto \N,$  $\eta(s) \doteq \sum_{i=1}^s i \mathds{1}_{i \mid s}$ Then, with $\Delta_i=\nu^{k^*}_{m^*,b^*} - \nu^{i}_{m_i^*,b_i^*}$
$$\regret \le \sum_{1 \le i \le K} \left( (1+ \eta(s) )\Delta_i + \frac{8 \alpha s \log(t)}{\Delta_i}\right).$$
\end{proposition}
\begin{proof}
The main difference with the standard technique to proving regret bounds comes from the fact that the value of each arm is the maximum of its coordinate: as suchs we have to consider the following event
\begin{align*}
\psi_{t}^{k^*} &\le\nu^{k^*}_{m^*,b^*} \\
\exists m+b \mid s,\quad\nu^{k}_{m^*_k,b^*_k}& \le \frac{s}{m+b} \left( \widehat{\psi}_{m,b,t-1}^k -\sqrt{\frac{2 \alpha (m+b) \log(t)}{s \tau_{k,t-1}}} \right)\\
\exists m+b \mid s,\quad  2\sqrt{\frac{2 \alpha (m+b) \log(t)}{s \tau_{k,t-1}}}& \ge \nu^{k^*}_{m^*,b^*} - \nu^{k}_{m_k^*,b_k^*}
\end{align*}
By carefully using the property of the maximum, the result can be recovered. All the details of the proof can be found in the supplementary material.
\end{proof}

We have just seen another approach to the rested mixing bandit. 
By assuming the summability of the $\varphi_k$, and by properly defining the value of an arm, we were able using Algorithm \ref{alg:blockUCBvar} to address the case where $m$ and $b$ are no longer fixed, but must be computed from the data by the agent.  It is interesting to note that the upper bound in Proposition \ref{pr:regretms} differs from the usual bound in the classical dependence-free setting by two multiplicative constants:  
$\eta(s)$, which encodes the total number of combination of pair $(m,b)$ such that $m+b \mid s$, and $s$, which is in fact used as an upper bound for $s/(m+b)$. 


%% file: restless.tex

\section{Restless $\varphi$ mixing bandits}\label{sec:restless}
In this section, we  provide an analysis for the restless $\varphi$-mixing bandit. Recall that contrarily to the rested case studied previously, the stochastic processes associated to each arm evolves regardless of the actions of the agent. This difference is of paramount importance in the $\varphi$-mixing setting. Indeed, in the rested case, the agent was bound to ignore some realizations obtained from an arm to enforce the independence and therefore the accuracy of its predictor. In the restless case, instead of pulling an arm to no avail, an agent willing to increase the independence of the realization of an arm $k$ can pull another arm  $k' \neq k$ to gather information about $k'$ while enforcing the independence of $k$. This idea is central to this section.

Like in Section~\ref{sec:vectorbandit}, we assume that $\forall k=1,\cdots,K$, the stochastic process $X^k$ is a $\varphi$-mixing process, and its mixing coefficients $\varphi^k(i)$ are summable, and we define the following upper bound function $\m_k$, which differs for the one defined in the previous section:
$$\forall 1\le k \le K,\quad \forall b \in \N ,\quad  \m_k(b)\doteq 1+ \sum_{i\ge b} \varphi_k (i) < + \infty .$$

In the restless $\varphi$-mixing bandit, the agent pulls the arm $k$ in sequences of $m_k$ trials, where the $m_k$ are fixed parameters and may differ for each arm. The mean value of this sequence is defined as follows:
$$\mu^k\doteq\expectation_{X_1^k,\ldots,X_{m_k}^k}\psi^k_m(X_1^k,\ldots,X_{m_k}^k)$$
and we use the same definition of regret as defined in Section~\ref{sec:easybandits}. In the restless setting, an interesting way of dealing with the trade-off exploration/exploitation/independence appears:  in addition to the usual exploration, it might be interesting for the agent to pull an apparently sub-optimal arm to get an increased independence on the result of the other arms --since the time between two consecutive sequences of pull decrease their dependency. In order to study this trade-off, we introduce a suitable concentration inequality.
\begin{proposition}
\label{pr:concentration restless}
Let $1 \le k \le K$, $1 \le m \le s$, $b=m-s$, $n \in \Nspace^*$, $\psi_m^k:\mathcal{U}^m\rightarrow\realset$ be a $1$-Lipschitz with respect to the Hamming metric function defined over a countable space $\mathcal{U}$. Let $\zeta_m^s: \Nspace \mapsto \Nspace$, be such that $$\zeta_m^s(t) =  t +b  \mathds{1}_{t \ge m+1} .$$ Then the following holds
for all $t>0$:
\begin{equation}
\label{eq:concentration restless}
\proba\left[\left| \frac{1}{n}\sum_{i=0}^{n-1} \psi_m ^k\left(X^k_{\zeta_m^s(im+1)},\cdots,X^k_{\zeta_m^s((i+1)m)}\right) -\mu^k\right|>t\right]\leq 2\exp\left[-\frac{n t^2}{2 \m_k^2(b)}\right].
\end{equation}
\end{proposition}
\begin{proof}
This proposition follows from the proof of Theorem \ref{th:kontorovitch} using the same technique as Proposition \ref{pr:concentration}.
\end{proof} 

 \begin{algorithm}[t]
\caption{\label{alg:restlessUCB}Restless Block-UCB with parameters $m_k$, $\alpha$ fixed}
\begin{algorithmic}
\State $t\leftarrow 0$, $\widehat{\psi}^k_0\leftarrow 0,\;$ $\tau_{k,0}\leftarrow 0,\;$ , $\eta_{k,0}\leftarrow 0,\;\quad k=1,\ldots,K$ 
\For{$t=1,\ldots,\tau$}
\State Select arm 
$ \displaystyle\widehat{k}\in\argmax_k \widehat{\psi}_{t-1}^k+\sqrt{\frac{2 \alpha \log(t)}{ \tau_{k,t-1} \m_k^2(\eta_{k,t})}}$
\State Update counters and timers :
$$\eta_{\widehat{k},t} \leftarrow 0, \tau_{\widehat{k},t}\leftarrow \tau_{\widehat{k},t-1}+1,$$ $$\forall k \neq \widehat{k},\quad 
\tau_{k,t}\leftarrow \tau_{k,t-1}\text{, and }\eta_{k,t} \leftarrow \eta_{k,t-1} + m_{\widehat{k}}
$$

\State Compute the values of the $\widehat{\psi}_{t}^k$, $m+b \mid s$
\begin{align*}
\widehat{\psi}_{t}^{\widehat{k}}&\leftarrow\frac{1}{\tau_{\widehat{k},t}}\sum_{r=0}^{\tau_{\widehat{k},t}-1}\psi_m^{\widehat{k}}\left(X_{ rm_{\widehat{k}}}^{\widehat{k}},\ldots,X_{(r+1)m_{\widehat{k}} -1}^{\widehat{k}}\right)\\
\widehat{\psi}_{t}^k&\leftarrow \widehat{\psi}_{t-1}^k,\; \text{for } k\neq \widehat{k}
\end{align*}
\EndFor
\end{algorithmic}
\end{algorithm}

It is interesting to note that the independence trade-off naturally appears in the right term of inequality~\ref{eq:concentration restless} within the $\m_k$, and will modify the upper confidence bound. We introduce algorithm~\ref{alg:restlessUCB} to solve this particular setting of $\varphi$ mixing bandit problem, and we provide a regret analysis for this algorithm. 

\begin{proposition}[Regret analysis]\label{pr:regretrestless}
Let $R(t)$ be the regret at time $t$ for Algorithm \ref{alg:restlessUCB}. Then, with $\Delta_i=\mu^{k^*}- \mu^{i},$
$$R_t \le \sum_{1 \le i \le K} \left( \Delta_i + \frac{8 \alpha \log(t)}{\Delta_i}\right).$$
\end{proposition}
\begin{proof}
The proof uses the same ideas as the previous regret analysis from Section 3 and 4, and naturally follows from Proposition \ref{pr:concentration restless}. The proof can be found in the supplementary materials. 
\end{proof}

In this section, we have provided an algorithm for the restless $\varphi$-mixing framework. We have seen that the restless setting is a natural framework to use with $\varphi$-mixing bandit, as it naturally makes it possible to decrease the dependence of the variables for the arms not chosen. Our algorithm takes advantage of this observation and we were able to show that is has low regret.
 In order to do so, in addition to the usual $\tau_k$, the number of time a given arm has been pulled, the Algorithm \ref{alg:restlessUCB} computes the $\eta_k$, the time spent since the last time the arm $k$ was sampled. Indeed, as seen in \eqref{eq:concentration restless}, as $\eta_k$ increases, $\m^2_k(\eta_k)$ decreases and the optimistic value of the arm increases.

%% file: conclusion.tex

\section{Conclusion}
\label{sec:conclusion}
We have studied an extension of the multi-armed bandit problem to the stationary $\varphi$-mixing framework, both in the rested and in the restless case. We have provided both a theoretical analysis in a general framework, and a more practical study of the problem in the case of fast mixing sequences (with $\sum\varphi(i)<+\infty$). For each of theses cases, we provided algorithms and accompanying regret analyses, which are strict extensions of the methods that exist for the i.i.d situation, as usual results might be recovered from our bounds when the mixing coefficients are all $0$. Future works might include a study of the restless case where the $m_k$ has to be computed from the data, as well as a study in the more difficult case of $\beta$-mixing processes. 


%% file: appendix.tex

\section{Appendix}

\begin{proposition}
\label{prop:stationary_subsequence}
Let $m, b \in \N$ and $s=m+b$, $X_t$ be a $\varphi$-mixing process on $\Omega$ taking value in $\R$, with mixing coefficient $\varphi_X(\cdot)$ , and and $\psi : \R^m \mapsto \R$ be a measurable function. Then the stochastic process $Z_t$ defined by

$$Z_t=\psi(X_{st+1}, \ldots, X_{st+m})$$

is also a $\varphi$-mixing process with mixing coefficient $\varphi_Z=\varphi_X \circ \kappa$, where $\kappa(t)= bt+ m(t-1)$.
\end{proposition}

\begin{proof}
In the following we use $\sigma(A)$ to denote the $\sigma$-algebra generated by $A$.
First note that since $\psi$ is measurable
$\sigma(\psi^{-1}(\R)) \subset \sigma (\R^m) $, and as a consequence 
$$\sigma(Z_t) \subset \sigma(X_{st+1}, \ldots, X_{st+m})$$ (since $\sigma$-algebra are closed under countable intersection) .

Now for any $i,j\in\mathbb{Z}\cup\{-\infty,+\infty\}$, let $\sigma_i^j(Z)$ denote
the $\sigma$-algebra generated by the random variables $Z_k$, $i\leq k\leq j$.
 Then, for any positive
integer $k$, the $\varphi$-mixing coefficient $\varphi_Z(t)$ of the stochastic process $\bfZ$ is defined as
\begin{align*}
\varphi_Z(t)&=\sup_{n,A\in\sigma_{n+t}^{+\infty}(Z),B\in\sigma_{-\infty}^n(Z)}\left|\proba\left[A|B\right]-\proba\left[A\right]\right|\\
& \le \sup_{n,A\in\sigma_{s(n+t)}^{+\infty}(X),B\in\sigma_{-\infty}^{ns+m}(X)}\left|\proba\left[A|B\right]-\proba\left[A\right]\right|\\
&=\varphi_X\left(s(n+t) - ns - m\right) = \varphi_X(tb+\left(t-1)m\right)
\end{align*} 

\end{proof}

\section{Proof of Proposition 2}

\bpf 

If at time $t$, the arm chosen is $k$ instead of $k^*  $, then one of the following must be true :

\begin{align}
\label{eq:proofregret1}\psi_{t}^{k^*} &\le\nu^{k^*}_{m^*,b^*} \\
\label{eq:proofregret2}\exists m+b \mid s,\quad\nu^{k}_{m^*_k,b^*_k}& \le \beta_{m,b} \left( \widehat{\psi}_{m,b,t-1}^k -\sqrt{\alpha} J_{m,b}(k,t-1) \right)\\
\label{eq:proofregret3}\exists m+b \mid s,\quad  2\sqrt{\frac{2 \alpha \beta_{m,b} \log(t)}{\tau_{k,t-1}}}& \ge \nu^{k^*}_{m^*,b^*} - \nu^{k}_{m_k^*,b_k^*}
\end{align}

Otherwise, $\forall m+b \mid s,$ 
 
 \begin{align*}
 \psi_{t}^{k^*} &\ge\nu^{k^*}_{m^*,b^*} \ge \nu^{k}_{m_k^*,b_k^*} + 2\sqrt{\frac{2 \alpha \beta_{m,b} \log(t)}{\tau_{k,t-1}}} \\
 &\ge \nu^{k}_{m_k^*,b_k^*} + 2\sqrt{\frac{2 \alpha \beta_{m,b} \log(t)}{\tau_{k,t-1}}}\\
 &\ge \beta_{m,b} \left( \widehat{\psi}_{m,b,t-1}^k +\sqrt{\alpha} J_{m,b}(k,t-1) \right)
 \end{align*}
 
 Since the last line is true $\forall m+b \mid s,$, we deduce that  $\psi_{t}^{k^*} \ge  \psi_{t}^{k}$ hence $T(t)\neq k$, which is absurd.
 
Then, we need to bound the probability of the events defined by \eqref{eq:proofregret1}, \eqref{eq:proofregret2} and \eqref{eq:proofregret3}.
 
 Because $\psi^{k^*}_t$ is defined as a maximum, we have
 \begin{align*}
 \P(\psi_{t}^{k^*} &\le\nu^{k^*}_{m^*,b^*}) \le \P\left(\beta_{m^*,b^*} ( \widehat{\psi}_{m,b,t-1}^k -\sqrt{\alpha} J_{m^*,b^*}(k^*,t-1) ) \le\nu^{k^*}_{m^*,b^*}\right)\\
 &\le \frac{1}{t^\alpha},
 \end{align*}
using  Proposition 1.

Now, by definition of $\nu^{k}_{m^*_k,b^*_k}$,
\begin{align*}
&\P\left(\exists m+b \mid s,\quad\nu^{k}_{m^*_k,b^*_k} \le \beta_{m,b} ( \widehat{\psi}_{m,b,t-1}^k -\sqrt{\alpha} J_{m,b}(k,t-1)) \right)\\
&\quad\le \P\left(\exists m+b \mid s,\quad\nu^{k}_{m,b} \le \beta_{m,b} ( \widehat{\psi}_{m,b,t-1}^k -\sqrt{\alpha} J_{m,b}(k,t-1)) \right)\\
& \quad\le \sigma(s) \max_{m+b \mid s} \P\left(\nu^{k}_{m,b} \le \beta_{m,b} ( \widehat{\psi}_{m,b,t-1}^k -\sqrt{\alpha} J_{m,b}(k,t-1)) \right)\\
& \le \frac{ \sigma(s)}{t^\alpha} 
\end{align*}
where we used Proposition 1 again at the last line.

Finally, 

\begin{align*}
&\left\lbrace\exists m+b \mid s,\quad  2\sqrt{\frac{2 \alpha \beta_{m,b} \log(t)}{\tau_{k,t-1}}} \ge \nu^{k^*}_{m^*,b^*} - \nu^{k}_{m_k^*,b_k^*} \right\rbrace \subset \left\lbrace 2\sqrt{\frac{2 \alpha s \log(t)}{\tau_{k,t-1}}} \ge \nu^{k^*}_{m^*,b^*} - \nu^{k}_{m_k^*,b_k^*} \right\rbrace\\
&= \left\lbrace \frac{8 \alpha s \log(t)}{(\nu^{k^*}_{m^*,b^*} - \nu^{k}_{m_k^*,b_k^*})^2} \ge \tau_{k,t-1} \right\rbrace
\end{align*}

i.e. the event defined by \eqref{eq:proofregret3} happens at most $ \displaystyle u=\left\lceil \frac{8 \alpha s \log(t)}{(\nu^{k^*}_{m^*,b^*} - \nu^{k}_{m_k^*,b_k^*})^2}\right\rceil$ times.

hence the conclusion
\epf
\section{Proof of Proposition 5}
\begin{proof}

If at time $t$, the arm chosen is $k$ instead of $k^*  $, then one of the following must be true :

\begin{align}
\psi_{t}^{k^*} &\le\mu^{k^*} \\
\mu^{k}& \le \left( \widehat{\psi}_{t-1}^k -\sqrt{\alpha} J(k,t-1) \right)\\
 2\sqrt{\frac{2 \alpha \log(t)}{\m_k^2(\eta_{k,t})  \tau_{k,t-1}}}& \ge \mu^{k^*} - \mu^{k}
\end{align}

From the last inequality, we deduce that :

\begin{align}
  \tau_{k,t-1} \le \frac{8 \alpha \log(t)}{\m_k^2(\eta_{k,t})  \Delta_k^2}& \le  \frac{8 \alpha \log(t)}{\Delta_k^2}
\end{align}
since $\m_k^2(\eta_{k,t})\le 1$.

\end{proof}